\documentclass[twoside,11pt]{article}

\usepackage[preprint]{jmlr2e}
\usepackage{amsmath}

\renewenvironment{proof}[1][Proof]{\par\noindent{\bf #1\ }}{\hfill\BlackBox\\[2mm]}

\ShortHeadings{Can an MLP Absorb Its Skip Connection Exactly?}{Mijoski and Karbevski}
\firstpageno{1}

\newcommand{\Wup}{W_{\mathrm{up}}}
\newcommand{\Wdn}{W_{\mathrm{down}}}
\newcommand{\Vup}{V_{\mathrm{up}}}
\newcommand{\Vdn}{V_{\mathrm{down}}}
\newcommand{\Wgate}{W_{\mathrm{gate}}}
\newcommand{\Wval}{W_{\mathrm{val}}}
\newcommand{\Vgate}{V_{\mathrm{gate}}}
\newcommand{\Vval}{V_{\mathrm{val}}}
\newcommand{\MLP}{\mathrm{MLP}}
\newcommand{\relu}{\mathrm{ReLU}}
\newcommand{\gelu}{\mathrm{GELU}}
\newcommand{\Fskip}{\mathcal{F}_{\mathrm{skip}}}
\newcommand{\Fff}{\mathcal{F}_{\mathrm{ff}}}
\newcommand{\R}{\mathbb{R}}
\newcommand{\diag}{\operatorname{diag}}
\newcommand{\rank}{\operatorname{rank}}
\newcommand{\sign}{\operatorname{sign}}

\begin{document}

\title{Can an MLP Absorb Its Skip Connection Exactly?}

\author{\name Antonij Mijoski \email antonijmijo@gmail.com, antonij.mijoski@unistra.fr \\
       \addr IRMA, Universit\'e de Strasbourg, France
       \AND
       \name Marko Karbevski \email marko.karbevski@gmail.com \\
       \addr In Simplicity Technologies}

\maketitle
\thispagestyle{plain}

\begin{abstract}%
The benefits usually attributed to skip connections are optimization-theoretic: a smoother loss landscape and better gradient propagation. We ask a representational question instead: given a residual block $x\mapsto x+\MLP(x)$, does a residual-free MLP of the same width compute the same function? The answer is no, unconditionally and at every depth, for every activation used in current frontier language models ($\relu^2$, ReGLU, SwiGLU, GeGLU). For ungated ReLU and GELU absorption is possible, but only on a set of weights of measure zero. The two families are therefore generically disjoint: removing a skip connection and retraining cannot recover exactly the same function at equal width.
\end{abstract}

\section{Introduction}

Skip connections were the breakthrough that enabled systematic training of deep networks \citep{he2016deep,he2016identity}. Every frontier architecture uses them or a variant: the Transformer \citep{vaswani2017}, Vision Transformers \citep{dosovitskiy2021}, ConvNeXt \citep{liu2022}, LLaMA \citep{touvron2023,dubey2024}, DeepSeek \citep{deepseek2024}, and Gemma \citep{gemma2024} wrap their sub-layers in residual branches $x\mapsto x+f(x)$, while hyper-connections (HC) \citep{zhu2024} and manifold-constrained hyper-connections (mHC) \citep{xie2025} replace the identity with learnable mixing matrices. The prevailing account is optimization-theoretic: skips smooth the landscape \citep{li2018} and help gradients propagate \citep{he2016identity,xiong2020}.

We ask whether skip connections change what an MLP can represent. Fix $\MLP_2(x)=\Wdn\,\sigma(\Wup x)$. When does there exist $\MLP_1(x)=\Vdn\,\sigma(\Vup x)$ of the same width with
\begin{equation}\label{eq:absorb}
x+\MLP_2(x)=\MLP_1(x)\qquad\forall\,x\in\R^d\,?
\end{equation}
Throughout, these MLPs carry no bias, matching modern transformers such as LLaMA, Gemma and Mistral. We do not treat the biased case, which we leave to future work. If such absorption is possible, the skip is representationally redundant at that width. If not, it grants access to a different region of function space. The converse question is equally natural: rearranging $\MLP_1(x)=x+\MLP_2(x)$ as $(-I_d)x+\MLP_1(x)=\MLP_2(x)$ shows both directions are instances of the same problem, and our results settle both (Appendix~\ref{sec:generic}).

\paragraph{Relation to universal approximation.}
\citet{lin2018} showed that a very deep ResNet with a single neuron per hidden layer approximates any Lebesgue-integrable function on $\R^d$ in $L^1$, in sharp contrast to fully connected networks, which are not universal approximators at width at most $d$ \citep{lu2017,hanin2017}. That result concerns the bottleneck regime, where width is the binding constraint. Once each hidden layer has at least $d+1$ neurons, plain networks are already universal as depth grows, so the skip adds no approximation power there. Our setting is that regime in practice ($N=4d$; we assume only $N\ge d$), and our question is orthogonal: not whether a target is approximable, but whether one fixed function is computed \emph{exactly} without the skip at the same width. Two blocks can approximate each other to any tolerance while remaining distinct as functions.

\paragraph{Contributions.}
For homogeneous activations of degree $k\neq1$, absorption is unconditionally impossible (Proposition~\ref{prop:homog}), which covers $\relu^2$ and ReGLU. The same holds for gated activations whose gate is differentiable at the origin and vanishes there (Proposition~\ref{prop:gated}), covering SwiGLU and GeGLU. Both extend to arbitrary depth (Corollary~\ref{cor:depth}). For ungated ReLU and GELU we give a single if-and-only-if characterization valid for generic weights (Theorem~\ref{thm:absorb}) and deduce that the two families are generically disjoint (Corollary~\ref{cor:generic}), while $d$ extra hidden units generically suffice (Proposition~\ref{prop:width}). Proofs are in the appendices.

\section{Setup}

For $A\in\R^{d\times N}$ and $B\in\R^{N\times d}$ we write $A[:,S]$ for the submatrix of columns indexed by $S\subseteq\{1,\dots,N\}$ and $B[S,:]$ for the rows indexed by $S$; the coordinate projector $\Pi_S\in\{0,1\}^{N\times N}$ is diagonal with $(\Pi_S)_{ii}=1$ iff $i\in S$, so $A\,\Pi_S B=A[:,S]\,B[S,:]$. A collection $\{a_1,\dots,a_N\}\subset\R^d$ is \emph{pairwise non-collinear} if $a_i$ is not a scalar multiple of $a_j$ for $i\neq j$. For univariate $f$ and $v\in\R^N$ we write $f(v)$ for coordinatewise application. We consider single-hidden-layer MLPs $\MLP(x)=\Wdn\,\sigma(\Wup x)$ with $\Wup\in\R^{N\times d}$, $\Wdn\in\R^{d\times N}$, $N\ge d$, and no bias terms.

\paragraph{Problem setup.}
The answer to \eqref{eq:absorb} differs sharply between the two architectures below. For gated MLPs absorption fails for \emph{every} choice of weights, so no measure-theoretic language is needed. For ungated MLPs it succeeds for some weights and fails for others, so we equip weight space with a measure and ask how large the successful set is.

A \emph{single-hidden-layer gated MLP} is a map $x\mapsto\Wdn\bigl(g(\Wgate x)\odot(\Wval x)\bigr)$ with $\Wgate,\Wval\in\R^{N\times d}$, $\Wdn\in\R^{d\times N}$, gate activation $g:\R\to\R$, and $\odot$ the Hadamard product. It uses three weight matrices and does not reduce to the ungated form. SwiGLU ($g=\mathrm{SiLU}$) and GeGLU ($g=\gelu$) are instances \citep{shazeer2020}. In the ungated case $(\Wup,\Wdn)$ ranges over $\Theta:=\R^{N\times d}\times\R^{d\times N}$, equipped with Lebesgue measure, and for fixed $\sigma$ parametrizes two families $\Fskip:=\{x\mapsto\Wdn\,\sigma(\Wup x)+x\}$ and $\Fff:=\{x\mapsto\Vdn\,\sigma(\Vup x)\}$; \eqref{eq:absorb} asks whether $\Fskip\subseteq\Fff$.

\paragraph{Even and odd parts.}
Both ungated proofs split the activation into even and odd parts. Writing $\relu(z)=\tfrac12(z+|z|)$ and, for the GELU, using $\Phi(-z)=1-\Phi(z)$ to define $E(z):=z(\Phi(z)-\tfrac12)$, we obtain
\begin{equation}\label{eq:parity}
\relu(z)=\underbrace{|z|/2}_{\text{even}}+\underbrace{z/2}_{\text{odd}},\qquad
\gelu(z)=\underbrace{E(z)}_{\text{even}}+\underbrace{z/2}_{\text{odd}}.
\end{equation}
In both cases the odd part is exactly $z/2$, and the proofs of Theorem~\ref{thm:absorb} rely on this.

\section{Homogeneous and gated activations: unconditional impossibility}\label{sec:homog}

Both results hold for a general linear skip $Mx$ with $M\neq0$, at no extra cost, so we state them that way; $M=I_d$ is the case we need.

\begin{proposition}[Homogeneous activations]\label{prop:homog}
Let $\sigma:\R\to\R$ satisfy $\sigma(\lambda z)=\lambda^k\sigma(z)$ for all $\lambda>0$, with $k>0$ and $k\neq1$. Then for any $N\ge d\ge1$, any $\Wup,\Vup\in\R^{N\times d}$, any $\Wdn,\Vdn\in\R^{d\times N}$, and any $M\in\R^{d\times d}$ with $M\neq0$, we have $Mx+\Wdn\,\sigma(\Wup x)\neq\Vdn\,\sigma(\Vup x)$ for some $x$.
\end{proposition}

This covers $\relu^2$, used in Nemotron \citep{nvidia2025} and modded-nanogpt \citep{jordan2024}, and ReGLU, whose neurons compute $\relu(a_i^{\top}x)\cdot(b_i^{\top}x)$ and are homogeneous of degree~2, as the proof uses only that each MLP map is positively homogeneous. Other gated activations are not homogeneous, but a linearization argument gives the same conclusion.

\begin{proposition}[Gated activations]\label{prop:gated}
Let $g:\R\to\R$ be differentiable at the origin with $g(0)=0$. Then for any $M\neq0$ and any gated MLPs with weights $\Wdn,\Wgate,\Wval$ and $\Vdn,\Vgate,\Vval$,
\[
Mx+\Wdn\bigl(g(\Wgate x)\odot(\Wval x)\bigr)\neq\Vdn\bigl(g(\Vgate x)\odot(\Vval x)\bigr)
\]
for some $x\in\R^d$.
\end{proposition}

Both SiLU and GELU vanish at the origin and are differentiable there, so SwiGLU and GeGLU are covered. Both impossibility results extend to arbitrary depth.

\begin{corollary}[Depth extension]\label{cor:depth}
Let all MLPs below be positively homogeneous of a common degree $k\neq1$, $k>0$, or all gated with $g$ differentiable at the origin and $g(0)=0$. Let $F=f_L\circ\cdots\circ f_1$ with $f_i(x)=x+\MLP_i(x)$, and $H=h_L\circ\cdots\circ h_1$ with $h_i(x)=\MLP'_i(x)$, all of width $N$. Then $F\neq H$.
\end{corollary}

This settles the question at every depth for every activation in current frontier language models: $\relu^2$, ReGLU, SwiGLU, and GeGLU.

\section{Ungated ReLU and GELU: an exact characterization}\label{sec:ungated}

Ungated activations behave differently, because the odd part $z/2$ in \eqref{eq:parity} contributes a linear term that a sign flip can turn against the skip. Write $w_i$ for the $i$-th row of $\Wup$.

\begin{theorem}[Identity absorption, ReLU and GELU]\label{thm:absorb}
Let $d\ge2$, $N\ge d$, and $\sigma\in\{\relu,\gelu\}$. Suppose $\Wup$ has pairwise non-collinear, non-zero rows and $\Wdn$ has non-zero columns. Then there exist $\Vup\in\R^{N\times d}$ and $\Vdn\in\R^{d\times N}$ with
\[
\Wdn\,\sigma(\Wup x)+x=\Vdn\,\sigma(\Vup x)\qquad\forall\,x\in\R^d
\]
if and only if there exists $S\subseteq\{1,\dots,N\}$ with $|S|\ge d$ such that
\begin{equation}\label{eq:cond}
\Wdn[:,S]\,\Wup[S,:]=-I_d.
\end{equation}
\end{theorem}

The two proofs differ in method, since $|z|$ is not differentiable at the origin while $E$ is analytic, yet they reach the same condition.

\paragraph{Where the condition comes from.}
Both proofs show any absorbing pair satisfies $\Vdn\Vup=\Wdn\,C\,\Wup$ with $C=\diag(\pm1)$, so, up to positive row scaling for ReLU, $\Vup$ is $\Wup$ with some rows sign-flipped. Flipping row $i$ leaves the even part untouched, since $E(-w_i^{\top}x)=E(w_i^{\top}x)$ and $\lvert -w_i^{\top}x\rvert=|w_i^{\top}x|$, while the odd part changes by $-w_i^{\top}x$. Collecting flipped rows in $S:=\{i:c_i=-1\}$, their contribution is $-\Wdn[:,S]\Wup[S,:]x$, which must equal $+x$. The factor~2 from the flip cancels the $1/2$ in the odd part, which is why \eqref{eq:cond} has $-I_d$ rather than a multiple of it; comparing ranks then forces $|S|\ge d$. Note that sign flips are not a symmetry here: they change the function, which is exactly why they can cancel a skip. The genuine gauge group of a generic GELU MLP is the hidden-unit permutations \citep{wang2025}.

\begin{corollary}[Generic disjointness]\label{cor:generic}
Let $d\ge2$, $N\ge d$, and $\sigma\in\{\relu,\gelu\}$. For Lebesgue-almost every $(\Wup,\Wdn)\in\Theta$ there exist no $\Vup,\Vdn$ with $\Wdn\,\sigma(\Wup x)+x=\Vdn\,\sigma(\Vup x)$ for all $x$.
\end{corollary}

Condition \eqref{eq:cond} imposes $d^2$ polynomial constraints, and there are only finitely many subsets $S$, so the absorbing locus is a finite union of proper subvarieties of $\Theta$ and has measure zero. The condition constrains $W$ alone, so the statement is the strong one: for almost every $W$, absorption fails for \emph{every} $V$, not merely for almost every pair.

\section{Discussion}

\paragraph{The role of width.}
For ReLU and GELU, condition \eqref{eq:cond} stays measure zero even at $N\gg d$: width adds candidate subsets $S$ without relaxing the $d^2$ equations any single subset must satisfy. Extra residual-free width does help: as $\sigma(z)-\sigma(-z)=z$, splitting $d$ neurons into opposed pairs adds an arbitrary linear map, so width $N+d$ generically absorbs the skip and, by the counting argument of \citet[Lemma~B.3]{grillo2026}, $N+d-1$ does not (Proposition~\ref{prop:width}). The results of Section~\ref{sec:homog} hold for unequal widths too. \citet{huang2018} read a residual network as a sum of weak learners, one per block, which is available only because each block has the form $x+\MLP(x)$. Our result says that, at equal width, this additive form is a property of the function rather than of how it is written.

\paragraph{Gauge symmetry without normalization.}
\citet{wang2025} prove the Transformer gauge group factors over layers, $G_{\mathrm{Model}}=\prod_\ell G^{(\ell)}_{\mathrm{Block}}$, using a non-equivariance lemma for LayerNorm: for generic gain and bias, $\mathrm{LN}(Mx)=M\,\mathrm{LN}(x)$ forces $M=I$. Normalization is removable in practice: \citet{baroni2025} fine-tune it out of GPT-2 up to XL at about $+0.03$ cross-entropy for GPT-2 XL, and \citet{graef2024} simplifies skipless transformers without it. Without normalization the factorization fails. Since $Q\bigl(x+\MLP(x)\bigr)=Qx+(Q\Wdn)\,\sigma\bigl((\Wup Q^{-1})\,Qx\bigr)$, and likewise for attention, any $Q\in\mathrm{GL}(d)$ applied to the residual stream of every layer, embedding and unembedding included, preserves the function: the skip carries $Q$ through all layers rather than blocking it. Corollary~\ref{cor:generic} concerns representability, not symmetry: gauge transformations keep the skip.

\paragraph{Trainability versus representation.}
In vision transformers, \citet{ji2025always} observe that self-attention fails to train without a skip while the feedforward sub-layer degrades only gracefully, and \citet{ji2025cutting} train skipless ones stably through a principled initialization, concluding skips are not fundamental for training; \citet{lin2018} show a skip restores universal approximation in the bottleneck regime. Our result is orthogonal to both: at equal width the two classes are generically disjoint, so a trained skipless block computes a different function rather than recovering the original.

\paragraph{Open problems.}
Depth is settled only for the activations of Section~\ref{sec:homog}. For ungated ReLU and GELU, a composition of residual-free blocks can realise a non-trivial linear map at the origin, so the Jacobian argument gives nothing; that case stays open.

\section*{Author Contributions}
MK conceived the research problem and formulated the question of exact skip-connection absorption. AM developed the theoretical foundation for ungated architectures, introducing the parity-decomposition framework and proving both the exact absorption characterization and the necessity conditions for ReLU. MK extended this framework to smooth activations to resolve the GELU characterization, and proved the unconditional impossibility results for homogeneous and gated networks across arbitrary depth. Both authors contributed to writing, refining the mathematical arguments, and preparing the manuscript.

\appendix

\section{Proofs for Section~\ref{sec:homog}}

\begin{proof}[Proof of Proposition~\ref{prop:homog}]
Suppose the identity holds for all $x$. Evaluating at $\lambda x$ for $\lambda>0$ and using homogeneity gives $\lambda Mx+\lambda^k\Wdn\,\sigma(\Wup x)=\lambda^k\Vdn\,\sigma(\Vup x)$. When $k>1$, dividing by $\lambda$ and letting $\lambda\to0$ gives $Mx=0$ for all $x$, contradicting $M\neq0$. When $0<k<1$, divide by $\lambda^k$ and let $\lambda\to\infty$ for the same contradiction.
\end{proof}

\begin{proof}[Proof of Proposition~\ref{prop:gated}]
Since $g(0)=0$ and $g$ is differentiable at $0$, we have $g(z)=g'(0)z+o(z)$ as $z\to0$. Each gated neuron satisfies
\[
g(w_{\mathrm{gate},i}^{\top}x)\cdot(w_{\mathrm{val},i}^{\top}x)=g'(0)\,(w_{\mathrm{gate},i}^{\top}x)(w_{\mathrm{val},i}^{\top}x)+o(\|x\|)\cdot O(\|x\|),
\]
so the full gated MLP output is $O(\|x\|^2)$ as $x\to0$. In the absorption equation both MLP terms are $O(\|x\|^2)$, hence $Mx=O(\|x\|^2)$. Since $Mx$ is linear, $M=0$.
\end{proof}

\begin{proof}[Proof of Corollary~\ref{cor:depth}]
In the gated case and for $k>1$, $\MLP(x)=o(\|x\|)$ near the origin, being $O(\|x\|^k)$ with $k>1$ in the homogeneous case and $O(\|x\|^2)$ in the gated case, so each MLP has vanishing Jacobian at $0$. Hence the Jacobian of $f_i$ at the origin is $I_d$ and that of $F$ is $I_d$, while the Jacobian of each $h_i$ at the origin is $0$ and that of $H$ is $0$. These cannot agree.
For $0<k<1$, continuity and homogeneity give $\|\MLP(y)\|\le C\|y\|^k$, so by induction over the blocks $F(\lambda x)=\lambda x+O(\lambda^k)$ as $\lambda\to\infty$. Meanwhile $H$ is positively homogeneous of degree $k^L<1$, so $H(\lambda x)/\lambda\to0$ while $F(\lambda x)/\lambda\to x$. Taking $x\neq0$ shows $F\neq H$.
\end{proof}

\section{Proofs for Section~\ref{sec:ungated}}

Throughout, $w_i$ denotes the $i$-th row of $\Wup$ and $v_j$ the $j$-th row of $\Vup$.

\subsection{Sufficiency (both activations)}

Suppose $S$ satisfies \eqref{eq:cond}. Set $D:=I_N-2\Pi_S$, a diagonal matrix of signs, and put $\Vup:=D\,\Wup$ and $\Vdn:=\Wdn$. Let $\sigma$ have parity decomposition $\sigma(z)=E_\sigma(z)+z/2$ with $E_\sigma$ even, as in \eqref{eq:parity}. Then
\begin{align*}
\Vdn\,\sigma(\Vup x)&=\Wdn\,E_\sigma(D\,\Wup x)+\tfrac12\Wdn\,D\,\Wup x\\
&=\Wdn\,E_\sigma(\Wup x)+\tfrac12\Wdn\,\Wup x-\Wdn\,\Pi_S\,\Wup x=\Wdn\,\sigma(\Wup x)+x,
\end{align*}
where the second equality uses evenness of $E_\sigma$ (the signs on the diagonal of $D$ are $\pm1$) and the third uses \eqref{eq:cond}.

\subsection{Necessity for ReLU}

Applying \eqref{eq:parity} to both sides of the absorption equation and separating the even and odd parts, which are uniquely determined, gives
\begin{equation}\label{eq:relu}
\Vdn\,|\Vup x|=\Wdn\,|\Wup x|,\qquad\Vdn\Vup=\Wdn\Wup+2I_d.
\end{equation}
The following lemma handles the first equation.

\begin{lemma}\label{lem:abs}
Let $d\ge2$ and let $\ell_1,\dots,\ell_M:\R^d\to\R$ be nonzero linear forms with pairwise non-collinear directions. If $\sum_{m=1}^M c_m|\ell_m(x)|=0$ for all $x$, then $c_m=0$ for every $m$.
\end{lemma}

\begin{proof}
Fix $r$. Since the directions are pairwise non-collinear, $\ker\ell_r\not\subseteq\ker\ell_m$ for $m\neq r$, so we may choose $x_0$ with $\ell_r(x_0)=0$ and $\ell_m(x_0)\neq0$ for all $m\neq r$, since the hyperplane $\ker\ell_r$ is not a finite union of the proper subspaces $\ker\ell_r\cap\ker\ell_m$. Pick $e$ with $\ell_r(e)\neq0$ and set $\varphi(t):=\sum_m c_m|\ell_m(x_0+te)|$, so $\varphi\equiv0$. For $m\neq r$ the function $t\mapsto|\ell_m(x_0+te)|$ is differentiable at $t=0$, because $\ell_m(x_0)\neq0$. For $m=r$ we have $|\ell_r(x_0+te)|=|t|\,|\ell_r(e)|$. Hence $0=\varphi'(0^+)-\varphi'(0^-)=2c_r|\ell_r(e)|$, so $c_r=0$.
\end{proof}

Group the $2N$ terms of the first equation of \eqref{eq:relu} by collinearity class, merging within each class, which is legitimate since $|\alpha\,\ell|=|\alpha|\,|\ell|$. Applying Lemma~\ref{lem:abs} row by row forces the merged coefficient of every class to vanish. The $w_i$ are pairwise non-collinear, so they occupy $N$ distinct classes, one each. Writing the $v$'s of the class of $w_i$ as $v_j=s_j\,w_i/\|w_i\|$, that class reads $\sum_j(\Vdn)_{\ell,j}|s_j|=(\Wdn)_{\ell,i}\|w_i\|$, whose right side is non-zero for some $\ell$ because $\Wdn$ has non-zero columns. Each of the $N$ classes therefore contains a $v_j$ with non-zero column, and since $\Vup$ has only $N$ rows (a zero row lies in no class), each contains exactly one and no class consists of $v$'s alone. This last point matters: a pair $v=\pm\hat u$ with opposite columns cancels in \eqref{eq:relu} while contributing a rank-one term to the second equation, so such classes must be excluded rather than discarded. Writing $v_j=c_iw_i$ we get $(\Vdn)_{:,j}|c_i|=(\Wdn)_{:,i}$, hence, after relabeling, $\Vdn\Vup=\sum_i\sign(c_i)\,(\Wdn)_{:,i}\,w_i^{\top}=\Wdn\,C\,\Wup$ with $C=\diag(\sign c_i)$.

Substituting into the second equation of \eqref{eq:relu} gives $\Wdn\,C\,\Wup=\Wdn\Wup+2I_d$, hence $\Wdn(C-I_N)\Wup=2I_d$. The diagonal entries of $C-I_N$ are $0$ or $-2$, so setting $S:=\{i:c_i<0\}$ gives $C-I_N=-2\Pi_S$ and therefore $\Wdn\,\Pi_S\,\Wup=-I_d$. Since $\rank(AB)\le\min(\rank A,\rank B)$, we get $d=\rank(I_d)\le|S|$.

\subsection{Necessity for GELU}

The even part of the GELU is smooth, so Lemma~\ref{lem:abs} does not apply. We use the following instead.

\begin{lemma}\label{lem:analytic}
Let $\psi:\R\to\R$ be analytic and not a polynomial, and let $\alpha_1,\dots,\alpha_m>0$ be pairwise distinct. If $\sum_k c_k\psi(\alpha_kt)=p(t)$ for all $t$ and some polynomial $p$, then $c_k=0$ for all $k$.
\end{lemma}

\begin{proof}
Write $\psi(t)=\sum_{n\ge0}a_nt^n$; since $\psi$ is not a polynomial, $a_n\neq0$ for infinitely many $n$. Then $\sum_kc_k\psi(\alpha_kt)=\sum_na_n\bigl(\sum_kc_k\alpha_k^n\bigr)t^n$. For this to equal a polynomial of degree at most $D$ we need $a_n\sum_kc_k\alpha_k^n=0$ for all $n>D$, hence $\sum_kc_k\alpha_k^n=0$ for infinitely many $n$. Choosing $m$ such exponents $n_1<\dots<n_m$ yields a generalized Vandermonde system with distinct positive bases, which is invertible. Hence $c_k=0$ for all $k$.
\end{proof}

Separating parities as before gives
\begin{equation}\label{eq:gelu}
\Vdn\,E(\Vup x)=\Wdn\,E(\Wup x),\qquad\Vdn\Vup=\Wdn\Wup+2I_d,
\end{equation}
with $E(z)=z(\Phi(z)-\tfrac12)$ analytic, even, and not a polynomial. Project the first equation onto coordinate $\ell$ and rearrange to $\sum_j(\Vdn)_{\ell,j}E(v_j^{\top}x)-\sum_i(\Wdn)_{\ell,i}E(w_i^{\top}x)=0$. Since $E$ is even, group the $2N$ terms by collinearity class of their directions and let $F_k$ denote the combined univariate function of the $k$-th class, so that $\sum_kF_k(a_k^{\top}x)=0$ with the class representatives $a_k$ pairwise non-collinear. By the uniqueness theorem for ridge functions \citep[Chapter~3]{pinkus2015}, see also \citet{buhmann1999}, each $F_k$ is a polynomial. Lemma~\ref{lem:analytic} applied to each $F_k$ shows that within a class the coefficient of each distinct positive scaling vanishes separately.

The $w_i$ occupy $N$ distinct classes, and in the class of $w_i$ the scaling $\|w_i\|$ carries the coefficient $(\Wdn)_{\ell,i}$, non-zero for some $\ell$. Scaling independence forces some $v_j$ in that class to have $\|v_j\|=\|w_i\|$, hence $v_j=\pm w_i$. As in the ReLU case the count closes the argument: $N$ classes each need such a $v_j$ and $\Vup$ has $N$ rows, so the matching is a bijection and no class consists of $v$'s alone. Each class therefore holds a single $v_j=\pm w_i$, and scaling independence gives $(\Vdn)_{:,j}=(\Wdn)_{:,i}$. Merged coefficients, not individual ones, are what Lemma~\ref{lem:analytic} annihilates, so ruling out $v$-only classes is needed rather than optional. After relabeling $\Vup=C\,\Wup$ with $C=\diag(\pm1)$ and $\Vdn=\Wdn$; extracting $S$ from the second equation of \eqref{eq:gelu} is verbatim as in the ReLU case.

\subsection{Proof of Corollary~\ref{cor:generic}}\label{sec:generic}

For $(\Wup,\Wdn)$ outside the measure-zero set where $\Wup$ has two collinear rows, a zero row, or $\Wdn$ has a zero column, Theorem~\ref{thm:absorb} applies and absorption holds if and only if \eqref{eq:cond} holds for some $S$. For each fixed $S$, condition \eqref{eq:cond} is a system of $d^2$ polynomial equations in the entries of $(\Wup,\Wdn)$, and it is not satisfied identically: at $\Wdn=0$ the left side is $0\neq-I_d$ for every $S$. Each such solution set is therefore a proper algebraic subvariety of $\Theta$ and has Lebesgue measure zero. There are $2^N$ subsets $S$, so the union over $S$ is still null.

\paragraph{Converse direction.} Negating both down-projections turns the absorption equation into $(-\Vdn)\,\sigma(\Vup x)+x=(-\Wdn)\,\sigma(\Wup x)$, an instance of Theorem~\ref{thm:absorb} for the given network $(\Vup,-\Vdn)$. Hence, if $\Vup$ has pairwise non-collinear, non-zero rows and $\Vdn$ has non-zero columns, $\Vdn\,\sigma(\Vup x)-x$ is a residual-free MLP of width $N$ if and only if ${\Vdn[:,S]\,\Vup[S,:]}=+I_d$ for some $S$. Negation preserves Lebesgue measure, so this fails for almost every $(\Vup,\Vdn)$.

\subsection{Width $N+d$}

\begin{proposition}[Width $N+d$]\label{prop:width}
Let $d\ge2$, $N\ge d$, and $\sigma\in\{\relu,\gelu\}$.
(i) If $\rank\Wup=d$, there exist $\Vup\in\R^{(N+d)\times d}$ and $\Vdn\in\R^{d\times(N+d)}$ with $\Wdn\,\sigma(\Wup x)+x=\Vdn\,\sigma(\Vup x)$ for all $x$.
(ii) For Lebesgue-almost every $(\Wup,\Wdn)\in\Theta$, no residual-free MLP of width at most $N+d-1$ computes $\Wdn\,\sigma(\Wup x)+x$.
\end{proposition}

Part~(ii) follows the counting argument of \citet[Lemma~B.3]{grillo2026}. For ReLU it is their lemma with $A=I_d$. The lemma is stated for generic hyperplane arrangements, which bias-free weights are not once $N>d$, but its proof applies unchanged. For GELU we run the same count, with Lemma~\ref{lem:analytic} in place of kink matching.

\begin{proof}
(i) Both activations satisfy $\sigma(z)-\sigma(-z)=z$, so splitting a neuron into an opposed pair adds a rank-one linear term: for any $u,\beta\in\R^d$,
\[
(u-\beta)\,\sigma(w^{\top}x)+\beta\,\sigma(-w^{\top}x)=u\,\sigma(w^{\top}x)-\beta\,w^{\top}x.
\]
Choose $K=\{i_1,\dots,i_d\}$ with $\Wup[K,:]$ invertible and let $\beta_k$ be the $k$-th column of $B:=-\Wup[K,:]^{-1}$. Splitting neuron $i_k$ with $\beta=\beta_k$ for $k=1,\dots,d$ gives width $N+d$, and the new network computes $\Wdn\,\sigma(\Wup x)-B\,\Wup[K,:]\,x$, which is $\Wdn\,\sigma(\Wup x)+x$.

(ii) Off a null set, $W$ satisfies the hypotheses of Theorem~\ref{thm:absorb}. Suppose a residual-free MLP of width $N'$ computes $\Wdn\sigma(\Wup x)+x$. The parity split gives \eqref{eq:relu} for ReLU and \eqref{eq:gelu} for GELU, whatever the widths. Write each nonzero $v_j$ as $s_ju_K$ with $u_K$ a unit vector of its collinearity class $K$, taking $u_K=w_i/\|w_i\|$ in the class of $w_i$, and set $P_K:=\|w_i\|(\Wdn)_{:,i}$ for that class and $P_K:=0$ for classes without a $w$. As in B.2 and B.3, the even parts give $\sum_{j\in K}|s_j|(\Vdn)_{:,j}=P_K$ for ReLU, and for GELU, by Lemma~\ref{lem:analytic} scale by scale, $\sum_{j\in K,\,|s_j|=\alpha}(\Vdn)_{:,j}$ equals $(\Wdn)_{:,i}$ if $K$ is the class of $w_i$ and $\alpha=\|w_i\|$, and $0$ otherwise. With $Q_K:=\sum_{j\in K}s_j(\Vdn)_{:,j}$, the odd parts give $\sum_K(Q_K-P_K)\,u_K^{\top}=2I_d$.

Call a class \emph{mixed} if it contains two $v$'s of opposite signs (for GELU, of opposite signs and equal $|s_j|$). In a class that is not mixed, the formulas above give $Q_K=P_K$ or $Q_K=-P_K$. Let $S$ collect the $i$ whose class is not mixed and has $Q_K=-P_K\neq0$. Then
\[
2\bigl(I_d+\Wdn[:,S]\,\Wup[S,:]\bigr)=\sum_{K\ \mathrm{mixed}}(Q_K-P_K)\,u_K^{\top}.
\]
Each of the $N$ classes of the $w$'s contains a $v$, since $P_K\neq0$, and each mixed class contains at least two $v$'s, so $N'$ is at least $N$ plus the number of mixed classes. If $N'\le N+d-1$, the right side therefore has rank below $d$, so $\det(I_d+\Wdn[:,S]\Wup[S,:])=0$. For each $S$ this determinant is a polynomial in $(\Wup,\Wdn)$ equal to $1$ at $\Wdn=0$, so it vanishes only on a null set, and there are finitely many $S$.
\end{proof}

\vskip 0.2in
\bibliography{references}

@article{baroni2025,
  author  = {Luca Baroni and Galvin Khara and Joachim Schaeffer and Marat Subkhankulov and Stefan Heimersheim},
  title   = {Transformers don't need {LayerNorm} at inference time: Scaling {LayerNorm} removal to {GPT-2 XL} and the implications for mechanistic interpretability},
  journal = {arXiv preprint arXiv:2507.02559},
  year    = {2025}
}

@article{buhmann1999,
  author  = {Martin D. Buhmann and Allan Pinkus},
  title   = {Identifying linear combinations of ridge functions},
  journal = {Advances in Applied Mathematics},
  volume  = {22},
  number  = {1},
  pages   = {103--118},
  year    = {1999}
}

@article{deepseek2024,
  author  = {{DeepSeek-AI}},
  title   = {{DeepSeek-V3} technical report},
  journal = {arXiv preprint arXiv:2412.19437},
  year    = {2024}
}

@inproceedings{dosovitskiy2021,
  author    = {Alexey Dosovitskiy and Lucas Beyer and Alexander Kolesnikov and Dirk Weissenborn and Xiaohua Zhai and Thomas Unterthiner and Mostafa Dehghani and Matthias Minderer and Georg Heigold and Sylvain Gelly and Jakob Uszkoreit and Neil Houlsby},
  title     = {An image is worth 16x16 words: Transformers for image recognition at scale},
  booktitle = {International Conference on Learning Representations ({ICLR})},
  year      = {2021}
}

@article{dubey2024,
  author  = {Abhimanyu Dubey and Abhinav Jauhri and Abhinav Pandey and others},
  title   = {The {Llama} 3 herd of models},
  journal = {arXiv preprint arXiv:2407.21783},
  year    = {2024}
}

@article{gemma2024,
  author  = {{Gemma Team}},
  title   = {Gemma: Open models based on {Gemini} research and technology},
  journal = {arXiv preprint arXiv:2403.08295},
  year    = {2024}
}

@article{graef2024,
  author  = {Nils Graef},
  title   = {Transformer tricks: Removing weights for skipless transformers},
  journal = {arXiv preprint arXiv:2404.12362},
  year    = {2024}
}

@article{grillo2026,
  author  = {Moritz Grillo and Guido Mont{\'u}far},
  title   = {Most {ReLU} networks admit identifiable parameters},
  journal = {arXiv preprint arXiv:2605.03601},
  year    = {2026}
}

@inproceedings{hanin2017,
  author    = {Boris Hanin and Mark Sellke},
  title     = {Approximating continuous functions by {ReLU} nets of minimal width},
  booktitle = {arXiv preprint arXiv:1710.11278},
  year      = {2017}
}

@inproceedings{he2016deep,
  author    = {Kaiming He and Xiangyu Zhang and Shaoqing Ren and Jian Sun},
  title     = {Deep residual learning for image recognition},
  booktitle = {Proceedings of the {IEEE} Conference on Computer Vision and Pattern Recognition ({CVPR})},
  pages     = {770--778},
  year      = {2016}
}

@inproceedings{he2016identity,
  author    = {Kaiming He and Xiangyu Zhang and Shaoqing Ren and Jian Sun},
  title     = {Identity mappings in deep residual networks},
  booktitle = {European Conference on Computer Vision ({ECCV})},
  pages     = {630--645},
  year      = {2016}
}

@inproceedings{huang2018,
  author    = {Furong Huang and Jordan T. Ash and John Langford and Robert E. Schapire},
  title     = {Learning deep {ResNet} blocks sequentially using boosting theory},
  booktitle = {Proceedings of the 35th International Conference on Machine Learning ({ICML})},
  pages     = {3272--3290},
  year      = {2018}
}

@article{ji2025cutting,
  author  = {Yiping Ji and James Martens and Jianqiao Zheng and Ziqin Zhou and Peyman Moghadam and Xinyu Zhang and Hemanth Saratchandran and Simon Lucey},
  title   = {Cutting the skip: Training residual-free transformers},
  journal = {arXiv preprint arXiv:2510.00345},
  year    = {2025}
}

@inproceedings{ji2025always,
  author    = {Yiping Ji and Hemanth Saratchandran and Peyman Moghadam and Simon Lucey},
  title     = {Always skip attention},
  booktitle = {Proceedings of the {IEEE/CVF} International Conference on Computer Vision ({ICCV})},
  year      = {2025}
}

@misc{jordan2024,
  author       = {Keller Jordan and Jeremy Bernstein and Brendan Rappazzo and {@fernbear.bsky.social} and Boza Vlado and Jiacheng You and Franz Cesista and Braden Koszarsky and {@Grad62304977}},
  title        = {modded-nanogpt: Speedrunning the {NanoGPT} baseline},
  howpublished = {\url{https://github.com/KellerJordan/modded-nanogpt}},
  year         = {2024}
}

@inproceedings{li2018,
  author    = {Hao Li and Zheng Xu and Gavin Taylor and Christoph Studer and Tom Goldstein},
  title     = {Visualizing the loss landscape of neural nets},
  booktitle = {Advances in Neural Information Processing Systems ({NeurIPS})},
  year      = {2018}
}

@inproceedings{lin2018,
  author    = {Hongzhou Lin and Stefanie Jegelka},
  title     = {{ResNet} with one-neuron hidden layers is a universal approximator},
  booktitle = {Advances in Neural Information Processing Systems 31 ({NeurIPS})},
  pages     = {6169--6178},
  year      = {2018}
}

@inproceedings{liu2022,
  author    = {Zhuang Liu and Hanzi Mao and Chao-Yuan Wu and Christoph Feichtenhofer and Trevor Darrell and Saining Xie},
  title     = {A {ConvNet} for the 2020s},
  booktitle = {Proceedings of the {IEEE/CVF} Conference on Computer Vision and Pattern Recognition ({CVPR})},
  pages     = {11966--11976},
  year      = {2022}
}

@inproceedings{lu2017,
  author    = {Zhou Lu and Hongming Pu and Feicheng Wang and Zhiqiang Hu and Liwei Wang},
  title     = {The expressive power of neural networks: A view from the width},
  booktitle = {Advances in Neural Information Processing Systems ({NeurIPS})},
  year      = {2017}
}

@article{nvidia2025,
  author  = {{NVIDIA}},
  title   = {Nemotron 3 {Nano}: Open, efficient mixture-of-experts hybrid {Mamba}-transformer model for agentic reasoning},
  journal = {arXiv preprint arXiv:2512.20848},
  year    = {2025}
}

@book{pinkus2015,
  author    = {Allan Pinkus},
  title     = {Ridge Functions},
  publisher = {Cambridge University Press},
  year      = {2015}
}

@article{shazeer2020,
  author  = {Noam Shazeer},
  title   = {{GLU} variants improve transformer},
  journal = {arXiv preprint arXiv:2002.05202},
  year    = {2020}
}

@article{touvron2023,
  author  = {Hugo Touvron and Thibaut Lavril and Gautier Izacard and Xavier Martinet and Marie-Anne Lachaux and Timoth{\'e}e Lacroix and Baptiste Rozi{\`e}re and Naman Goyal and Eric Hambro and Faisal Azhar and Aurelien Rodriguez and Armand Joulin and Edouard Grave and Guillaume Lample},
  title   = {{LLaMA}: Open and efficient foundation language models},
  journal = {arXiv preprint arXiv:2302.13971},
  year    = {2023}
}

@inproceedings{vaswani2017,
  author    = {Ashish Vaswani and Noam Shazeer and Niki Parmar and Jakob Uszkoreit and Llion Jones and Aidan N. Gomez and Lukasz Kaiser and Illia Polosukhin},
  title     = {Attention is all you need},
  booktitle = {Advances in Neural Information Processing Systems ({NeurIPS})},
  pages     = {5998--6008},
  year      = {2017}
}

@inproceedings{wang2025,
  author    = {Hong Wang and Kelly Wang},
  title     = {Complete characterization of gauge symmetries in transformer architectures},
  booktitle = {Symmetry and Geometry in Neural Representations ({NeurReps}), {NeurIPS} Workshop},
  year      = {2025},
  url       = {https://openreview.net/forum?id=KrkbYbK0cH}
}

@article{xie2025,
  author  = {Zhenda Xie and Yixuan Wei and Huanqi Cao and Chenggang Zhao and Chengqi Deng and Jiashi Li and Damai Dai and Huazuo Gao and Jiang Chang and Liang Zhao and Shangyan Zhou and Zhean Xu and Zhengyan Zhang and Wangding Zeng and Shengding Hu and Yuqing Wang and Jingyang Yuan and Lean Wang and Wenfeng Liang},
  title   = {{mHC}: Manifold-constrained hyper-connections},
  journal = {arXiv preprint arXiv:2512.24880},
  year    = {2025}
}

@inproceedings{xiong2020,
  author    = {Ruibin Xiong and Yunchang Yang and Di He and Kai Zheng and Shuxin Zheng and Chen Xing and Huishuai Zhang and Yanyan Lan and Liwei Wang and Tie-Yan Liu},
  title     = {On layer normalization in the transformer architecture},
  booktitle = {International Conference on Machine Learning ({ICML})},
  year      = {2020}
}

@article{zhu2024,
  author  = {Defa Zhu and Hongzhi Huang and Zihao Huang and Yutao Zeng and Yunyao Mao and Banggu Wu and Qiyang Min and Xun Zhou},
  title   = {Hyper-connections},
  journal = {arXiv preprint arXiv:2409.19606},
  year    = {2024}
}

\end{document}